\documentclass[sigconf]{acmart}

\usepackage{graphicx}
\graphicspath{ {./images/} }
\usepackage{amsmath}
\usepackage{amsthm}
\usepackage[ruled,vlined]{algorithm2e}
\usepackage{subfiles} 
\usepackage{enumitem}

\newcommand{\our}{PinnerSage\xspace}

\newcommand{\hide}[1]{}
\newcommand{\xhdr}[1]{\vspace{2mm}{\noindent\bfseries #1}.}

\copyrightyear{2020}
\acmYear{2020}
\setcopyright{rightsretained}
\acmConference[KDD '20]{Proceedings of the 26th ACM SIGKDD Conference on Knowledge Discovery and Data Mining}{August 23--27, 2020}{Virtual Event, CA, USA}
\acmBooktitle{Proceedings of the 26th ACM SIGKDD Conference on Knowledge Discovery and Data Mining (KDD '20), August 23--27, 2020, Virtual Event, CA, USA}
\acmDOI{10.1145/3394486.3403280}
\acmISBN{978-1-4503-7998-4/20/08}

\begin{document}
\fancyhead{}

\title{PinnerSage: Multi-Modal User Embedding Framework for Recommendations at Pinterest}

\author{Aditya Pal*, Chantat Eksombatchai*, Yitong Zhou*, Bo Zhao, Charles Rosenberg, Jure Leskovec}
\thanks{*These authors contributed equally}
\affiliation{%
  \institution{Pinterest Inc.}
}
\email{{apal, pong, yzhou, bozhao, crosenberg, jure}@pinterest.com}

\ccsdesc[500]{Information systems~Recommender systems}
\ccsdesc[500]{Information systems~Personalization}
\ccsdesc[300]{Information systems~Clustering}
\ccsdesc[300]{Information systems~Nearest-neighbor search}

\keywords{Personalized Recommender System; Multi-Modal User Embeddings}

\renewcommand{\shortauthors}{Aditya Pal, et al}

\begin{abstract}
Latent user representations are widely adopted in the tech industry for powering personalized recommender systems. Most prior work infers a \textit{single} high dimensional embedding to represent a user, which is a good starting point but falls short in delivering a full understanding of the user's interests. In this work, we introduce PinnerSage, an end-to-end recommender system that represents each user via multi-modal embeddings and leverages this rich representation of users to provides high quality personalized recommendations. PinnerSage achieves this by clustering users' actions into conceptually coherent clusters with the help of a hierarchical clustering method (Ward) and summarizes the clusters via representative pins (Medoids) for efficiency and interpretability. PinnerSage is deployed in production at Pinterest and we outline the several design decisions that makes it run seamlessly at a very large scale. We conduct several offline and online A/B experiments to show that our method significantly outperforms single embedding methods.
\end{abstract}

\maketitle

{\fontsize{8pt}{8pt} \selectfont
\textbf{ACM Reference Format:}\\
Aditya Pal, Chantat Eksombatchai, Yitong Zhou, Bo Zhao, Charles Rosenberg, Jure Leskovec. 2020. PinnerSage: Multi-Modal User Embedding Framework for Recommendations at Pinterest. In \textit{Proceedings of the 26th ACM SIGKDD Conference on Knowledge Discovery \& Data Mining (KDD '20), August 23--27, 2020, Virtual Event, CA, USA.} ACM, New York, NY, USA, 10 pages. https://doi/10.1145/3394486.3403280}

\section{Introduction}
Pinterest is a content discovery platform that allows 350M+ monthly active users to collect and interact with 2B+ visual bookmarks called \textit{pins}. Each pin is an image item associated with contextual text, representing an idea that users can find and bookmark from around the world. Users can save pins on \textit{boards} to keep them organized and easy to find. With billions of pins on Pinterest, it becomes crucial to help users find those ideas (Pins) which would spark inspiration. Personalized recommendations thus form an essential component of the Pinterest user experience and is pervasive in our products. 
The Pinterest recommender system spans a variety of algorithms that collectively define the experience on the homepage. Different algorithms are optimized for different objectives and include~-~
(a) homefeed recommendations where a user can view an infinite recommendation feed on the homepage (as shown in Figure~\ref{fig:pinterest}), (b) shopping product recommendations which link to 3rd party e-commerce sites, (c) personalized search results, (d) personalized ads, (e) personalized pin board recommendations, etc. Therefore, it becomes necessary to develop a universal, shareable and rich understanding of the user interests to power large-scale cross-functional use cases at Pinterest.

\begin{figure}[t]
\centering
\includegraphics[scale=0.37]{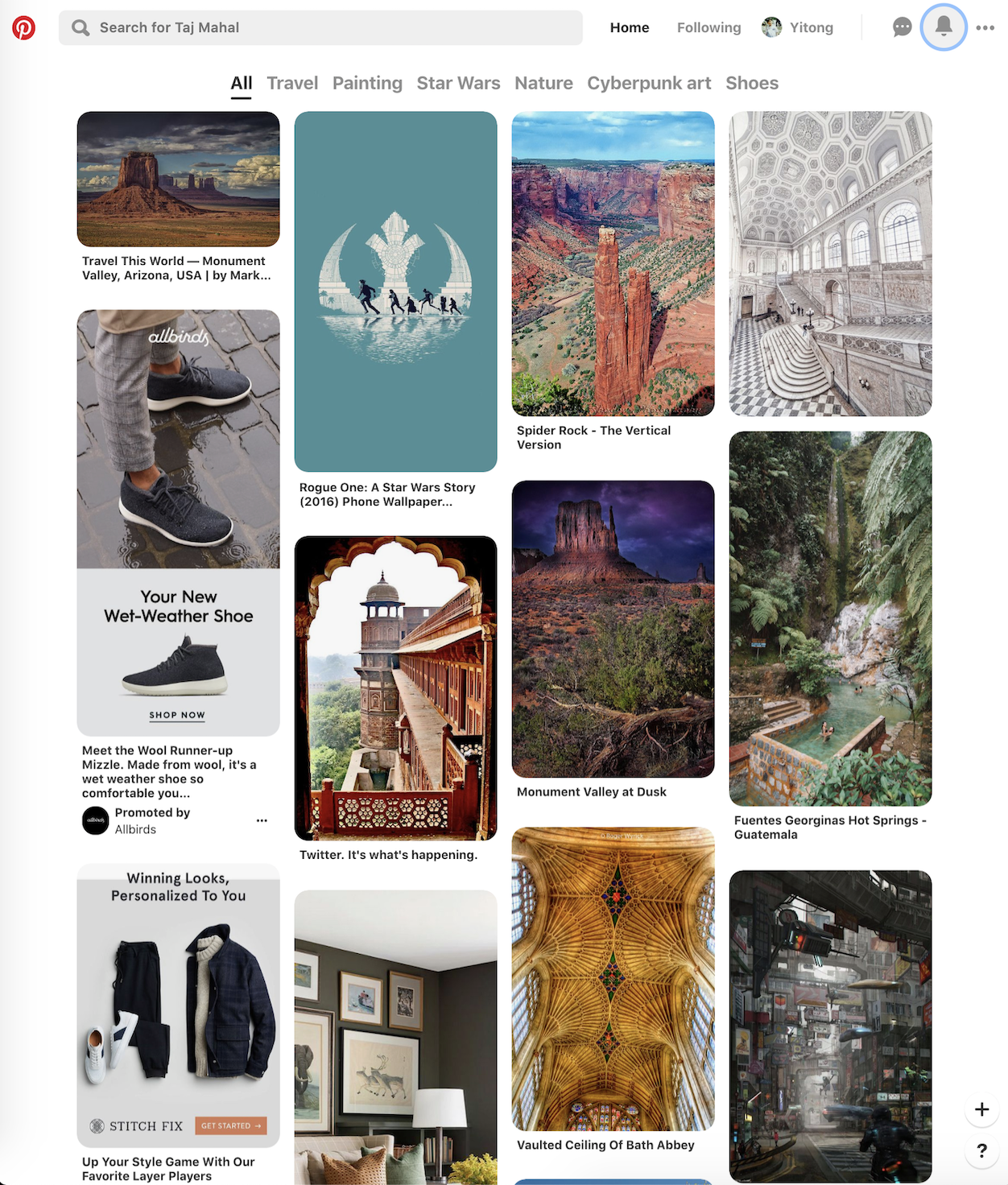}
\caption{Pinterest Homepage.}
\label{fig:pinterest}
\end{figure}

Latent user representation methods have become increasingly important in advancing our understanding of users. 
They are shown to be effective at powering collaborative filtering techniques \cite{itemcf, amazonitemcf}, 
and serving as features in ranking models~\cite{widedeep, starspace, dlyoutube, htcn}. 
Due to their efficacy, user embeddings are widely adopted in various industry settings.
They are used to power YouTube and Google play recommendations \cite{dlyoutube, widedeep}, 
personalize search ranking and similar listing recommendations at Airbnb \cite{airbnb}, 
recommend news articles to users \cite{yahoornn}, 
connect similar users at Etsy~\cite{etsy}, etc. 

Building an effective user embedding system that provides personalized recommendations to hundreds of millions of users from a candidate pool of billions of items has several inherent challenges.
\textit{The foremost challenge is how to effectively encode multiple facets of a user?}
A user can have a diverse set of interests with no obvious linkage between them.
These interests can evolve, with some interests persisting long term while others spanning a short time period.
Most of the prior work aims to capture the rich diversity of a user's actions and interests via a single high-dimensional embedding vector.
Typically, items to be recommended are also represented in the same embedding space.
This is initially satisfying, but as pointed by research work described in \cite{microprofile, westonmultiintr, alibabaKdd2019, epastomultipersona}, a good embedding must encode user's multiple tastes, interests, styles, etc., whereas an item (a video, an image, a news article, a house listing, a pin, etc) typically only has a single focus. 
Hence, an attention layer \cite{zhang2018attention} or other context adapting approaches is needed to keep track of the evolving interests of the users.

One alternative that has shown promise is to represent a user with multiple embeddings, with each embedding capturing a specific aspect of the user.
As shown by \cite{westonmultiintr}, multi-embedding user representations can deliver 25\% improvement in YouTube video recommendations.
\cite{alibabaKdd2019} also shows reasonable gains on small benchmark datasets.
However, multi-embedding models are not widely adopted in the industry due to several important questions and concerns that are not yet fully addressed by prior work:
\begin{itemize}
    \item How many embeddings need to be considered per user?
    \item How would one run inference at scale for hundreds of millions of users and update the embeddings ?
    \item How to select the embeddings to generate personalized recommendations?
    \item Will the multi-embedding models provide any significant gains in a production setting?
\end{itemize}
Most of the prior multi-embedding work side-steps these challenges by either running only small-scale experiments and not deploying these techniques in production or by limiting a user to very few embeddings, thereby restricting the utility of such approaches.

\xhdr{Present Work}
In this paper, we present an end-to-end system, called \our, that is deployed in production at Pinterest.
\our is a highly scalable, flexible and extensible recommender system that internally represents each user with multiple PinSage~\cite{pinsage} embeddings.
It infers multiple embeddings via hierarchical clustering of users' actions into conceptual clusters and uses an efficient representation of the clusters via medoids. Then, it employs a highly efficient nearest neighbor system to power candidate generation for recommendations at scale.
Finally, we evaluate \our extensively via offline and online experiments. 
We conduct several large scale A/B tests to show that PinnerSage based recommendations result in significant engagement gains for Pinterest's homefeed, and shopping product recommendations

\section{PinnerSage Design Choices}
\label{sec:designchoices}
We begin by discussing important design choices of \our.

\xhdr{Design Choice 1: Pin Embeddings are Fixed}
Most prior work, jointly learns user and item embeddings (e.g. ~\cite{dlyoutube, airbnb}).
This causes inherent problems in large-scale applications, such that it unnecessarily complicates the model, slows the inference computation, and brings in difficulty for real-time updates.
Besides these, we argue that it can often lead to less desirable side-effects.
Consider the toy example in Figure~\ref{fig:curse_of_single_embedding}, where a user is interested in \textit{painting}, \textit{shoes}, and \textit{sci-fi}.
Jointly learnt user and pin embeddings would bring pin embeddings on these disparate topics closer, which is actually what we wish to \textit{avoid}. Pin embeddings should only operate on the underlying principle of bringing similar pins closer while keeping the rest of the pins as far as possible.
For this reason, we use PinSage~\cite{pinsage}, which precisely achieves this objective without any dilution. 
PinSage is a unified pin embedding model, which integrates visual signals, text annotations, and pin-board graph information to generate high quality pin embeddings.
An additional advantage of this design choice is that it considerably simplifies our downstream systems and inference pipelines.


\begin{figure}[t]
\centering
\includegraphics[scale=0.23]{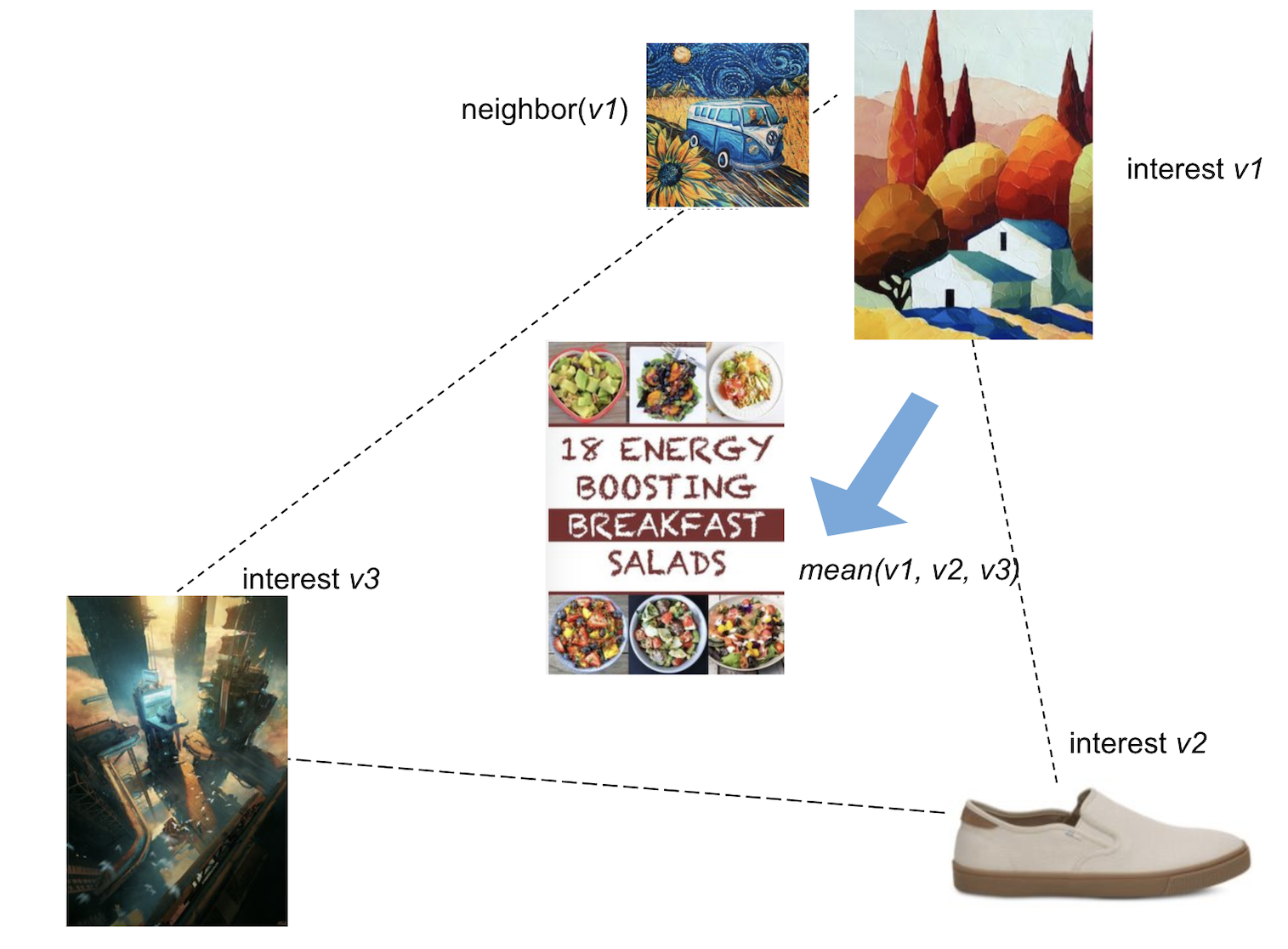}
\caption{Pins of 256-dimensional embeddings visualized in 2D. These pins depict three different interests of a user.}
\label{fig:curse_of_single_embedding}
\vspace{-0.2cm}
\end{figure}

\xhdr{Design Choice 2: No Restriction on Number of Embeddings}
Prior work either fixes the number of embeddings to a small number~\cite{westonmultiintr} or puts an upper bound on them~\cite{alibabaKdd2019}.
Such restrictions at best hinders developing a full understanding of the users and at worst merges different concepts together leading to bad recommendations. 
For example, merging item embeddings, which is considered reasonable (see~\cite{dlyoutube, htcn}), could yield an embedding that lies in a very different region.
Figure~\ref{fig:curse_of_single_embedding} shows that a merger of three disparate pin embeddings results in an embedding that is best represented by the concept \textit{energy boosting breakfast}.
Needless to say, recommendations based on such a merger can be problematic.

Our work allows a user to be represented by as many embeddings as their underlying data supports. 
This is achieved by clustering users' actions into conceptually coherent clusters via a hierarchical agglomerative clustering algorithm (Ward).
A light user might get represented by 3-5 clusters, whereas a heavy user might get represented by 75-100 clusters. 

\xhdr{Design Choice 3: Medoids based Representation of Clusters} Typically, clusters are represented by centroid, which requires storing an embedding. Additionally, centroid could be sensitive to outliers in the cluster.
To compactly represent a cluster, we pick a cluster member pin, called medoid.
Medoid by definition is a member of the user's originally interacted pin set and hence avoids the pit-fall of topic drift and is robust to outliers.
From a systems perspective, medoid is a concise way of representing a cluster as it only requires storage of medoid's pin id and also leads to cross-user and even cross-application cache sharing.

\xhdr{Design Choice 4: Medoid Sampling for Candidate Retrieval}
\our provides a rich representation of a user via cluster medoids. However, in practice we cannot use all the medoids simultaneously for candidate retrieval due to cost concerns. Additionally, the user would be bombarded with too many different items. Due to cost concerns, we sample 3 medoids proportional to their importance scores (computation described in later sections) and recommend their nearest neighboring pins.
The importance scores of medoids are updated daily and they can adapt with changing tastes of the user.

\xhdr{Design Choice 5: Two-pronged Approach for Handling Real-Time Updates}
It is important for a recommender system to adapt to the current needs of their users.
At the same time an accurate representation of users requires looking at their past 60-90 day activities.
Sheer size of the data and the speed at which it grows makes it hard to consider both aspects together.
Similar to~\cite{airbnb}, we address this issue by combining two methods: (a) a daily batch inference job that infers multiple medoids per user based on their long-term interaction history, and (b) an online version of the same model that infers medoids based on the users' interactions on the current day.
As new activity comes in, only the online version is updated.
At the end of the day, the batch version consumes the current day's activities and resolves any inconsistencies.
This approach ensures that our system adapts quickly to the users' current needs and at the same time does not compromise on their long-term interests.




\xhdr{Design Choice 6: Approximate Nearest Neighbor System}
To generate embeddings based recommendations, we employ an approximate nearest neighbor (ANN) system.
Given a query (medoid), the ANN system fetches $k$ pins closest to the query in the embedding space.
We show how several improvements to the ANN system, such as filtering low quality pins, careful selection of different indexing techniques, caching of medoids; results in the final production version having 1/10 the cost of the original prototype.



\section{Our Approach}
\label{section:method}

\xhdr{Notations}
Let the set $\mathcal{P} =\{P_1, P_2, \ldots\}$ represent all the pins at Pinterest. The cardinality of $\mathcal{P}$ is in order of billions.
Here, $P_i \in \mathbb{R}^d$ denotes the $d$-dimensional PinSage embedding of the $i^{th}$ pin.
Let $\mathcal{A}_u = \{a_1, a_2, \ldots \}$ be the sequence of action pins of user $u$, such that for each $a \in \mathcal{A}_u$, user either \textit{repinned}, or \textit{clicked} pin $P_{a}$ at time $\mathcal{T}_u[a]$.
For the sake of simplicity, we drop the subscript $u$ as we formulate for a single user $u$, unless stated otherwise. 
We consider the action pins in $\mathcal{A}$ to be sorted based on action time, such that $a_1$ is the pin id of the first action of the user.

\xhdr{Main Assumption: Pin Embeddings are Fixed}
As mentioned in our design choice 1 (Section~\ref{sec:designchoices}), we consider pin embeddings to be fixed and generated by a black-box model. 
Within Pinterest, this model is PinSage~\cite{pinsage} that is trained to place similar pins nearby in the embedding space with the objective of subsequent retrieval.
This assumption is ideal in our setting as it considerably simplifies the complexity of our models.
We also made a similar assumption in our prior work~\cite{htcn}.

\xhdr{Main Goal}
Our main goal is to infer multiple embeddings for each user, $\mathcal{E} = \{e_1, e_2, \ldots \}$, where $e_i \in \mathbb{R}^d$ for all $i$, given a user's actions $\mathcal{A}$ and pins embeddings $\mathcal{P}$.
Since pin embeddings are fixed and hence \textit{not} jointly inferred, we seek to learn each $e_i$ compatible with pin embeddings -- \textit{specifically for the purpose of retrieving similar pins to $e_i$}.
For different users, the number of embeddings can be different, i.e. $|\mathcal{E}_u|$ need \textit{not} be same as $|\mathcal{E}_v|$.
However, for our approach to be practically feasible, we \textit{require the number of embeddings to be in order of tens to hundreds} ($|\mathcal{E}| << |\mathcal{P}|$). \\

\noindent To show the promise of the clustering-based approach, we consider a task of predicting the next user action.
We have access to the user's past actions $\{a_1, a_2, \ldots, a_i\}$ and our goal is to predict the next action $a_{i+1}$ that the user is going to interact with from a corpus of billions of pins.
To simplify this challenge, we measure the performance by asking how often is the distance between the user-embedding and the pin embedding $P_{a_{i+1}}$ is above a cosine threshold of $0.8$.
We compare four single embedding approaches:
\begin{enumerate}
    \item \textit{Last pin}: User representation is the embedding of the last action pin ($P_{a_i}$).
    \item \textit{Decay average}: User representation is a time-decayed average of embeddings of their action pins. Specifically, decay average embedding $\propto \sum_{a \in \mathcal{A}} e^{-\lambda (\mathcal{T}_{now} - \mathcal{T}[a])} \cdot P_a$.
    \item \textit{Oracle}: Oracle can ``look into the future'' and pick as the user representation the past action pin of the user that is closest to $a_{i+1}$. This measures the upper bound on accuracy of a method that would predict future engagements based on past engagements.
    \item \textit{Kmeans Oracle}: User is represented via k-means clustering ($k=3$) over their past action pins. Again, the Oracle gets to see pin $a_{i+1}$ and picks as the user representation the cluster centroid closest to it.
\end{enumerate}

\begin{table}[t]
\begin{center}
\caption{Accuracy lift of models on predicting the next user action. The lifts are computed w.r.t. last pin model.}
\label{tab:oracleexp}
\begin{tabular}{l|c}
\textbf{Models} & \textbf{Accuracy Lift} \\
\hline
{Last pin} & 0\% \\
{Decay average} & 25\% \\
{Kmeans Oracle} & 98\% \\
{Oracle} (\textit{practical upper-bound}) & \textbf{140\%}
\end{tabular}
\end{center}
\end{table}

\begin{figure*}[t]
\centering
\includegraphics[scale=0.48]{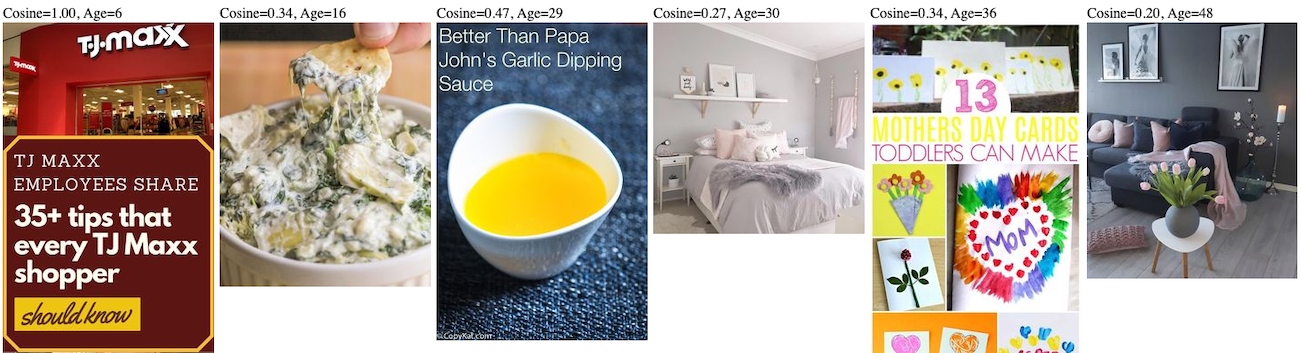}
\caption{Snapshot of action pins (\textit{repins} or \textit{clicks}) of a random user.
Cosine score for a pin is the cosine similarity between its embedding and the latest pin.
Age of a pin is the number of days elapsed from the action date to the data collection date.}
\label{fig:activities}
\end{figure*}

Table~\ref{tab:oracleexp} shows the results. Oracle model provides substantial accuracy gains, deservedly so as it can look at the future pin.
However, its superior performance is only possible because it is able to recall the embedding of all the pins (obviously not practical from a systems point-of-view).
Interestingly, a clustering based Oracle that only has to recall $3$ cluster centroid embeddings improves over the baselines by a large margin. 
This result is not entirely surprising because users have multiple interests and they often switch between those interests.
Figure~\ref{fig:activities} depicts such an example, which is replete in our setting.
We note that none of the past 5 pins correlate with the latest pin and one has to look further back to find stronger correlations.
Hence, single embedding models with limited memory fail at this challenge.

\subsection{PinnerSage}
We draw two key insights from the previous task: (i)
It is too limiting to represent a user with a single embedding, and (ii) Clustering based methods can provide a reasonable trade-off between accuracy and storage requirement.
These two observations underpin our approach, which has the following three components.
\begin{enumerate}
\item Take users' action pins from the last 90 days and cluster them into a small number of clusters.
\item Compute a medoid based representation for each cluster.
\item Compute an importance score of each cluster to the user.
\end{enumerate}


\subsubsection{\textbf{Step 1: Cluster User Actions}}
\label{section:clustering}
We pose two main constraints on our choice of clustering methods.
\begin{itemize}
    \item The clusters should \textit{only} combine conceptually similar pins.
    \item It should automatically determine the number of clusters to account for the varying number of interests of a user.
\end{itemize}
The above two constraints are satisfied by Ward~\cite{ward}, which is a hierarchical agglomerative clustering method, that is based on a minimum variance criteria (satisfying our constraint 1).
Additionally, the number of clusters is automatically determined based on the distances (satisfying our constraint 2).
In our experiments (sec.~\ref{section:experiment}), it performed better than K-means and complete linkage methods~\cite{clink}.
Several other benchmark tests also establish the superior performance of Ward over other clustering methods \footnote{\url{https://jmonlong.github.io/Hippocamplus/2018/02/13/tsne-and-clustering/}}.


\begin{algorithm}[t]
    \caption{Ward($\mathcal{A} = \{a_1, a_2, \ldots \}, \alpha$)}
    \label{algo:WardClustering}
    \SetKwInOut{Input}{Input}
    \SetKwInOut{Output}{Output}
    \Input{$\mathcal{A}$ - action pins of a user \\ $\alpha$ - cluster merging threshold}
    \Output{Grouping of input pins into clusters}

    // Initial set up: each pin belongs to its own cluster\\
    Set $C_i \leftarrow \{ i \}, \forall i \in \mathcal{A}$\\
    Set $d_{ij} \leftarrow ||P_i - P_j||^2_2, \forall i, j \in \mathcal{A}$\\
    $merge\_history = []$\\
    $stack = []$ \\
    \While{$|\mathcal{A}| > 1$} {
        // put first pin from $\mathcal{A}$ (without removing it) to the stack \\
        $stack.add(\mathcal{A}.first())$\\
        \While{$|stack| > 0$} {
        	   $i \leftarrow  stack.top()$\\
	   $J \leftarrow \{j : d_{ij} = \min_{j \neq i, j \in \mathcal{A}} \{d_{ij}\} \} $\\
	   $merged = False$ \\
	   \If{$|stack| \geq 2$} {
	   $j = stack.second\_from\_top()$ \\
	   \If {$j \in J$} {
	   	// merge clusters $i$ and $j$ \\
	   	$stack.pop()$; stack.pop(); //remove $i$ and $j$ \\
	   	merge\_history.add($C_i \leftarrow C_j, d_{ij}$) \\
	   	$\mathcal{A} \leftarrow \mathcal{A} - \{j\}$ \\
		Set $d_{ik} \leftarrow d(C_i \cup C_j, C_k)$, $\forall k \in \mathcal{A}, k \neq i$ \\
		$C_i \leftarrow C_i \cup C_j$ \\
		$merged = True$
	   }
	   }
	   
	   \If {$merged = False$} {
	   // push first element of $J$ in the stack \\
	   $stack.push(J.first())$ 
	   }
    }
}
Sort tuples in $merge\_history$ in decreasing order of $d_{ij}$
    
    Set $\mathbb{C} \leftarrow \{\}$\\
    \ForEach{$(C_i \leftarrow C_j, d_{ij}) \in $ merge\_history}{
    \If{$d_{ij} \leq \alpha$ and $ \{C_i \cup C_j\}  \cap C = \varnothing, \forall C \in \mathbb{C}$} {
    	// add $C_i \cup C_j$ to the set of clusters \\
	Set $\mathbb{C} \leftarrow \mathbb{C} \cup \{C_i \cup C_j\}$
    }
    }
\KwRet $\mathbb{C}$
\end{algorithm}

Our implementation of Ward is adapted from the Lance-Williams algorithm~\cite{lance67general} which provided an efficient way to do the clustering. 
Initially, Algorithm~\ref{algo:WardClustering} assigns each pin to its own cluster. 
At each subsequent step, two clusters that lead to a minimum increase in within-cluster variance are merged together.
Suppose after some iterations, we have clusters $\{C_1, C_2, \ldots \}$ with distance between clusters $C_i$ and $C_j$ represented as $d_{ij}$.
Then if two clusters $C_i$ and $C_j$ are merged, the distances are updated as follows:
\begin{equation}
 d(C_i \cup C_j, C_k) = \frac{(n_i + n_k) \ d_{ik} + (n_j + n_k) \ d_{jk} - n_k \ d_{ij}}{n_i + n_j + n_k} 
     \label{eq:wardminvar}
\end{equation}
where $n_i = |C_i|$ is the number of pins in cluster $i$.

\xhdr{Computational Complexity of Ward Clustering Algorithm}
The computational complexity of Ward clustering is $\mathbb{O}(m^2)$, where $m = |\mathcal{A}|^2$. 
To see this, we note that in every outer while loop a cluster is added to the \textit{empty} stack. 
Now since a cluster cannot be added twice to the stack (see Appendix, Lemma~\ref{lem:ward}), the algorithm has to start merging the clusters once it cycles through all the $m$ clusters (worst case). 
The step that leads to addition of a cluster to the stack or merging of two clusters has a computational cost of $\mathbb{O}(m)$.
The algorithm operates with $m$ \textit{initial clusters} and then $m-1$ \textit{intermediate merged clusters} as it progresses, leading to the total computational complexity of $\mathbb{O}(m^2)$.

\subsubsection{\textbf{Step 2: Medoid based Cluster Representation}}
After a set of pins are assigned to a cluster, we seek a compact representation for that cluster.
A typical approach is to consider cluster centroid, time decay average model or a more complex sequence models such as LSTM, GRU, etc.
However one problem with the aforementioned techniques is that the embedding inferred by them could lie in a very different region in the $d$-dimensional space.
This is particularly true if there are outlier pins assigned to the cluster, which could lead to large with-in cluster variance.
The side-effect of such an embedding would be retrieval of non-relevant candidates for recommendation as highlighted by Figure~\ref{fig:curse_of_single_embedding}.

We chose a more robust technique that selects a cluster member pin called as \textit{medoid} to represent the cluster.
We select the pin that minimizes the sum of squared distances with the other cluster members.
Unlike centroid or embeddings obtained by other complex models, medoid by definition is 
a point in the $d$-dimensional space that coincides with one of the cluster members.
Formally,
\begin{equation}
\mbox{$embedding(C) \leftarrow P_m$, where } m = \arg\min_{m \in C} \sum_{j \in C} ||P_m - P_j ||^2_2
\end{equation}
An additional benefit of medoid is that we only need to store the index $m$ of the medoid pin as its embedding can be fetched on demand from an auxiliary key-value store.

\begin{algorithm}[t]
    \caption{\our ($\mathcal{A}, \alpha, \lambda$)}
    \label{algo:pinnersage}
    \SetKwInOut{Input}{Input}
    \SetKwInOut{Output}{Output}
    Set $\mathbb{C} \leftarrow \mbox{Ward}(\mathcal{A}, \alpha)$ \\
    \ForEach{$C \in \mathbb{C}$}{
        Set $medoid_C \leftarrow \arg\min_{m \in C} \sum_{j \in C} ||P_m - P_j ||^2_2$ \\
        Set $importance_C \leftarrow \sum_{i \in C} e^{-\lambda (\mathcal{T}_{now} - \mathcal{T}[i])}$
    }
\KwRet $\{{medoid_C}, importance_C : \forall C \in \mathbb{C} \}$
\end{algorithm}

\subsubsection{\textbf{Step 3: Cluster Importance}} \label{model:clusterimportance}
Even though the number of clusters for a user is small, it can still be in order of tens to few hundreds. 
Due to infra-costs, we cannot utilize all of them to query the nearest neighbor system; making it essential to identify the relative importance of clusters to the user so we can sample the clusters by their importance score.
We consider a time decay average model for this purpose:
\begin{equation}
Importance(C, \lambda) = \sum_{i \in C} e^{-\lambda (\mathcal{T}_{now} - \mathcal{T}[i])} 
\end{equation}
where $\mathcal{T}[i]$ is the time of action on pin $i$ by the user.
A cluster that has been interacted with frequently and recently will have higher importance than others. 
Setting $\lambda=0$ puts more emphasis on the frequent interests of the user, whereas $\lambda=0.1$ puts more emphasis on the recent interests of the user. We found $\lambda=0.01$ to be a good balance between these two aspects.

Algorithm~\ref{algo:pinnersage} provides an end-to-end overview of \our model.
We note that our model operates independently for each user and hence it can be implemented quite \textit{efficiently in parallel} on a MapReduce based framework.
We also maintain an online version of \our that is run on the most recent activities of the user. 
The output of the batch version and the online version are combined together and used for generating the recommendations.

\section{PinnerSage Recommendation System}
\label{section:system}
\our can infer as many medoids for a user as the underlying data supports. This is great from a user representation point of view, however not all medoids can be used simultaneously at any given time for generating recommendations.
For this purpose, we consider importance sampling. 
We sample a maximum of $3$ medoids per user at any given time. The sampled medoids are then used to retrieve candidates from our nearest-neighbor system.
Figure~\ref{fig:system_graph} provides an overview of \our recommendation system.

\begin{figure}[t]
\centering
\includegraphics[scale=0.31]{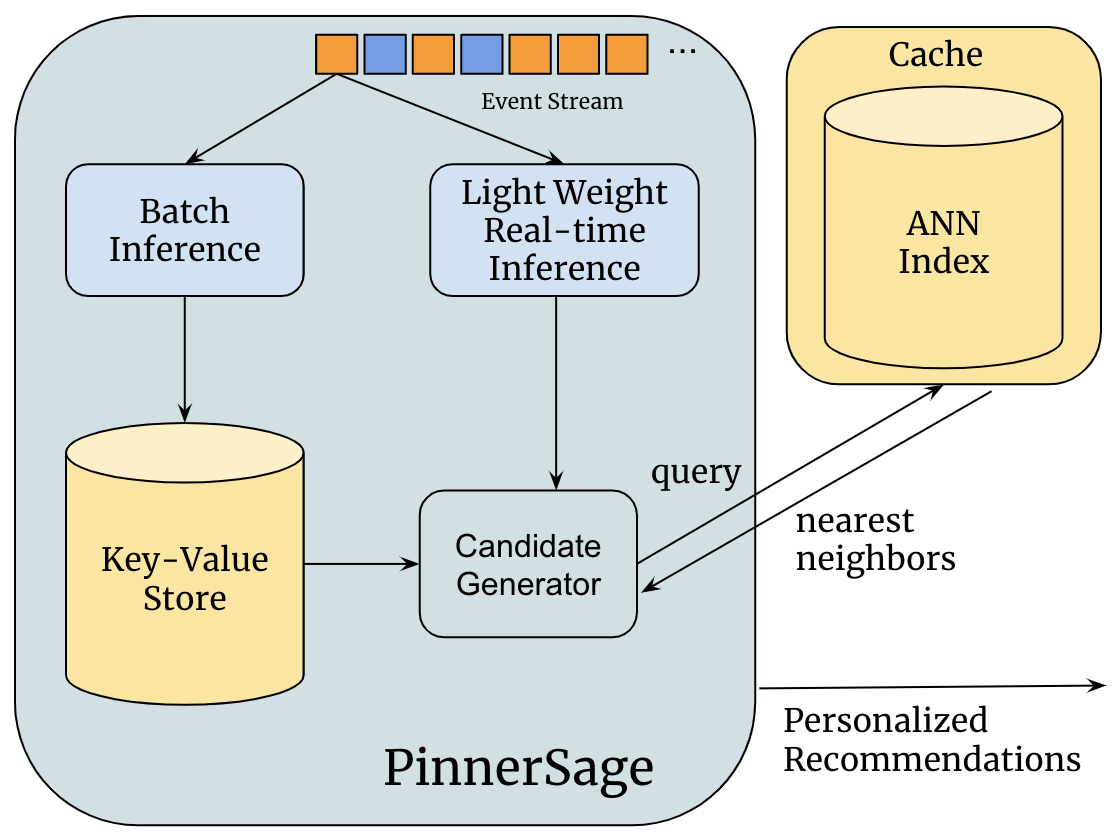}
\caption{PinnerSage Recommendation System.}
\label{fig:system_graph}
\end{figure}

\subsection{Approximate Nearest Neighbor Retrieval}
\label{section:system_nearest_neighbors}
Within Pinterest, we have an approximate nearest neighbor retrieval system (ANN) that maintains an efficient index of pin embeddings enabling us to retrieve similar pins to an input query embedding.
Since it indexes billions of pins, ensuring that its infrastructure cost and latency is within the internally prescribed limits is an engineering challenge.
We discuss a few tricks that have helped ANN become a first-class citizen alongside other candidate retrieval frameworks, such as Pixie~\cite{pixie}.

\subsubsection{\textbf{Indexing Scheme}}
Many different embedding indexing schemes (see~\cite{faiss}) were evaluated, such as LSH Orthoplex \cite{lsh, terasawa2007spherical}, Product Quantization \cite{babenko2016efficient}, HNSW \cite{malkov2018efficient}, etc.
We found HNSW to perform best on cost, latency, and recall.
Table~\ref{tab:system_opt_gains} shows that HNSW leads to a significant cost reduction over LSH Orthoplex. 
Superior performance of HNSW has been reported on several other benchmark datasets as well \footnote{\url{https://erikbern.com/2018/06/17/new-approximate-nearest-neighbor-benchmarks.html}}. 

\xhdr{Candidate Pool Refinement}
A full index over billions of pins would result in retrieving many near-duplicates.
These near duplicates are not desirable from recommendation purposes as there is limited value in presenting them to the user.
Furthermore, some pins can have intrinsically lower quality due to their aesthetics (low resolution or large amount of text in the image). 
We filter out near duplicates and lower quality pins via specialized in-house models.
Table~\ref{tab:system_opt_gains} shows that index refinement leads to a significant reduction in serving cost. 


\xhdr{Caching Framework}
All queries to the ANN system are formulated in pin embedding space.
These embeddings are represented as an array of $d$ floating point values that are not well suited for caching.
On the other-hand, medoid's pin id is easy to cache and can reduce repeated calls to the ANN system.
This is particularly true for popular pins that appear as medoids for multiple users.
Table~\ref{tab:system_opt_gains} shows the cost reduction of using medoids over centroids.

\begin{table}[t]
  \begin{center}
    \caption{Relative Cost Benefits of Optimization Techniques.}
    \label{tab:system_opt_gains}      
    \begin{tabular}{rcl|c}
        \multicolumn{3}{c|}{\textbf{Optimization Technique}} & \textbf{Cost} \\ \hline
      LSH Orthoplex & $ \rightarrow $ & HNSW & -60\% \\
      Full Index & $ \rightarrow $ & Index Refinement & -50\% \\
      Cluster Centroid & $ \rightarrow $ & Medoid & -75\%
    \end{tabular}
  \end{center}
\end{table}

\subsection{Model Serving}
\label{section:online_serving}
The main goal of \our is to recommend relevant content to the users based on their past engagement history.
At the same time, we wish to provide recommendations that are relevant to actions that a user is performing in the real-time.
One way to do this is by feeding all the user data to \our and run it as soon as user takes an action. However this is practically not feasible due to cost and latency concerns:
We consider a two pronged approach:
\begin{enumerate}
    \item \textit{Daily Batch Inference}: \our is run daily over the last 90 day actions of a user on a MapReduce cluster. The output of the daily inference job (list of medoids and their importance) are served online in key-value store.
    \item \textit{Lightweight Online Inference}: We collect the most recent 20 actions of each user on the latest day (after the last update to the entry in the key-value store) for online inference. \our uses a real-time event-based streaming service to consume action events and update the clusters initiated from the key-value store.
\end{enumerate}


In practice, the system optimization plays a critical role in enabling the productionization of \our. Table \ref{tab:system_opt_gains} shows a rough estimation of cost reduction observed during implementation. 
While certain limitations are imposed in the PinnerSage framework, such as a two pronged update strategy, the architecture allows for easier improvements to each component independently.

\section{Experiment}
\label{section:experiment}
Here we evaluate \our and empirically validate its performance.
We start with qualitative assessment of \our, followed by A/B experiments and followed by extensive offline evaluation.


\begin{figure*}[t]
\centering
\includegraphics[scale=0.44]{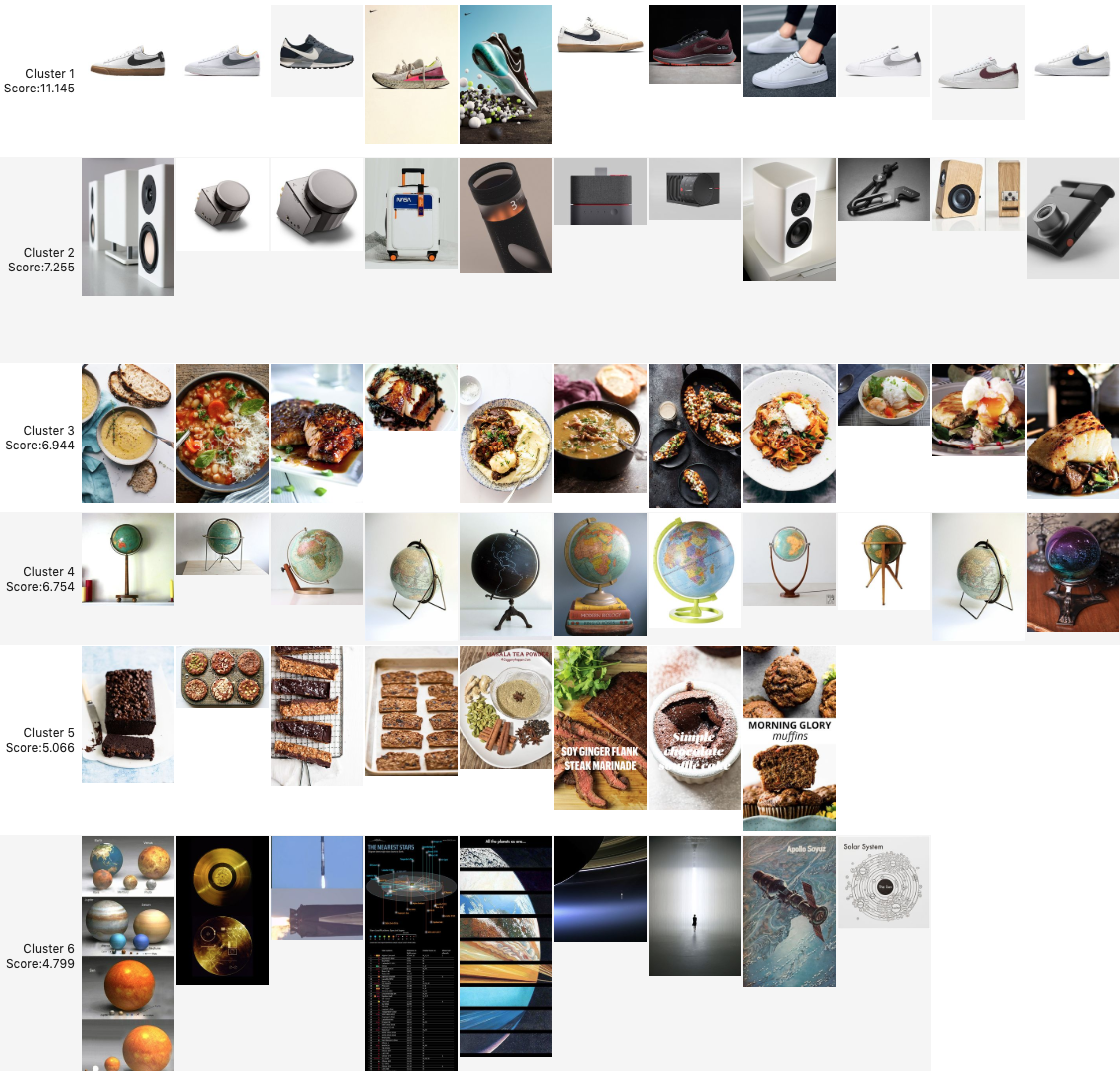}
\caption{\our clusters of an anonymous user.}
\label{fig:cluster_visualized}
\end{figure*}

\begin{figure*}[t]
\centering
\includegraphics[scale=0.35]{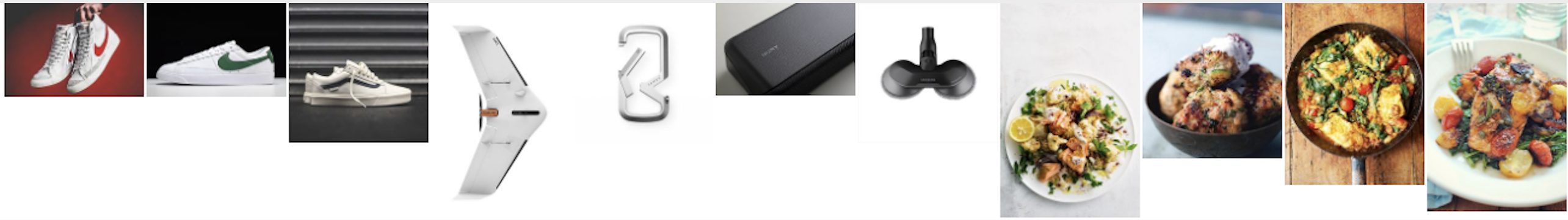}
\caption{Sample recommendations generated by \our for the top $3$ clusters~\ref{fig:cluster_visualized}.}
\label{fig:pinnersage_examples}
\end{figure*}

\subsection{\our Visualization}
Figure \ref{fig:cluster_visualized} is a visualization of \our clusters for a given user.
As can be seen, \our does a great job at generating conceptually consistent clusters by grouping only contextually similar pins together.
Figure \ref{fig:pinnersage_examples} provides an illustration of candidates retrieved by \our. 
The recommendation set is a healthy mix of pins that are relevant to the top three interests of the user: \textit{shoes}, \textit{gadgets}, and \textit{food}. Since this user has interacted with these topics in the past, they are likely to find this diverse recommendation set interesting and relevant. 

\begin{table}[t]
\begin{center}
\caption{A/B experiments across Pinterest surfaces. Engagement gain of \our vs current production system.}
\label{tab:abexp}
\begin{tabular}{l|c|c}
\textbf{Experiment} & \textbf{Volume} & \textbf{Propensity}\\ \hline
Homefeed  & +4\% & +2\% \\
Shopping  & +20\% & +8\% \\
\end{tabular}
\end{center}
\end{table}

\subsection{Large Scale A/B Experiments}
\label{section:online_experiments}
We ran large scale A/B experiments where users are randomly assigned either in control or experiment groups.
Users assigned to the experiment group experience \our recommendations, while users in control get recommendations from the single embedding model (decay average embedding of action pins). 
Users across the two groups are shown equal number of recommendations.
Table~\ref{tab:abexp} shows that \our provides significant engagement gains on increasing overall engagement volume (repins and clicks) as well as increasing engagement propensity (repins and clicks per user).
Any gain can be directly attributed to increased quality and diversity of \our recommendations.

\subsection{Offline Experiments}
\label{section:offline_experiments}
We conduct extensive offline experiments to evaluate \our and its variants w.r.t. baseline methods. 
We sampled a large set of users (\textit{tens of millions}) and collected their past activities (\textit{actions} and \textit{impressions}) in the last 90 days. 
All activities before the day $d$ are marked for training and from $d$ onward for testing.

\xhdr{Baselines} We compare \our with the following baselines: (a) single embedding models, such as last pin, decay avg with several choices of $\lambda$ ($0, 0.01, 0.1, 0.25$), LSTM, GRU, and HierTCN\cite{htcn}, (b) multi-embedding variants of \our with different choices of (i) clustering algorithm, (ii) cluster embedding computation methods, and (iii) parameter $\lambda$ for cluster importance selection.

Similar to \cite{htcn}, baseline models are trained with the objective of ranking user actions over impressions with several loss functions (l2, hinge, log-loss, etc). Additionally, we trained baselines with several types of negatives besides impressions, such as random pins, popular pins, and hard negatives by selecting pins that are similar to action pins.

\xhdr{Evaluation Method}
The embeddings inferred by a model for a given user are evaluated on future actions of that user.
Test batches are processed in chronological order, first day $d$, then day $d + 1$, and so on.
Once evaluation over a test batch is completed, that batch is used to update the models; mimicking a daily batch update strategy.

\subsubsection{\textbf{Results on Candidate Retrieval Task.}} \label{sec:evalret}
Our main use-case of user embeddings is to retrieve relevant candidates for recommendation out of a very large candidate pool (billions of pins).
The candidate retrieval set is generated as follows: Suppose a model outputs $e$ embeddings for a user, then $\lfloor \frac{400}{e}\rfloor$ nearest-neighbor pins are retrieved per embedding, and finally the retrieved pins are combined to create a recommendation set of size $\leq 400$ (due to overlap it can be less than 400).
The recommendation set is evaluated with the observed user actions from the test batch on the following two metrics:
\begin{enumerate}
\item \textbf{Relevance (\textit{Rel.})} is the proportion of observed action pins that have high cosine similarity ($\geq 0.8$) with any recommended pin.
Higher relevance values would increase the chance of user finding the recommended set useful. 
\item \textit{\textbf{Recall}} is the proportion of action pins that are found in the recommendation set.
\end{enumerate}

Table~\ref{tab:retrievaltask} shows that \our is more effective at retrieving relevant candidates across all baselines.
In particular, the single embedding version of \our is better than the state-of-art single embedding sequence methods.

Amongst \our variants, we note that Ward performs better than K-means and complete link methods. For cluster embedding computation, both sequence models and medoid selection have similar performances, hence we chose medoid as it is easier to store and has better caching properties.
Cluster importance with $\lambda=0$, which is same as counting the number of pins in a cluster, performs worse than $\lambda=0.01$ (our choice).
Intuitively this makes sense as higher value of $\lambda$ incorporates recency alongside frequency. 
However, if $\lambda$ is too high then it over-emphasize recent interests, which can compromise on long-term interests leading to a drop in model performance ($\lambda=0.1$ vs $\lambda=0.01$).

\begin{table}[t]
\caption{Lift relative to \textit{last pin model} for retrieval task.}
\label{tab:retrievaltask}
\begin{tabular}[t]{l|c|c}
& \textit{\textbf{Rel.}} & \textit{\textbf{Recall}} \\ \hline
Last pin model & 0\% & 0\% \\
Decay avg. model ($\lambda = 0.01$) & 28\% & 14\% \\
Sequence models (HierTCN) & 31\% & 16\% \\
\our(sample $1$ embedding) & 33\% & 18\% \\ \hline
\our(K-means(k=5)) & 91\% & 68\% \\
\our(Complete Linkage) & 88\% & 65\% \\
\our(embedding = Centroid) & 105\% & 81\% \\
\our(embedding = HierTCN) & \textbf{110\%} & \textbf{88\%} \\
\our(importance $\lambda=0$) & 97\% & 72\% \\
\our(importance $\lambda=0.1$) & 94\% & 69\% \\
\our(Ward, Medoid, $\lambda=0.01$) & \textbf{110\%} & \textbf{88\%}
\end{tabular}
\end{table}

\subsubsection{\textbf{Results on Candidate Ranking Task.}} 
\label{sec:evalrank}
A user embedding is often used as a feature in a ranking model especially to rank candidate pins.
The candidate set is composed of action and impression pins from the test batch. 
To ensure that every test batch is weighted equally, we randomly sample 20 impressions per action. In the case when there are less than 20 impressions in a given test batch, we add random samples to maintain the 1:20 ratio of actions to impressions.
Finally the candidate pins are ranked based on the decreasing order of their maximum cosine similarity with any user embedding.
A better embedding model should be able to rank actions above impressions.
This intuition is captured via the following two metrics:
\begin{enumerate}
\item \textbf{R-Precision (\textit{R-Prec.})} is the proportion of action pins in top-$k$, where $k$ is the number of actions considered for ranking against impressions. RP is a measure of signal-to-noise ratio amongst the top-$k$ ranked items.

\item \textbf{Reciprocal Rank (\textit{Rec. Rank})} is the average reciprocal rank of action pins. It measures how high up in the ranking are the action pins placed.
\end{enumerate}

Table~\ref{tab:rankingtask} shows that \our significantly outperforms the baselines, indicating the efficacy of user embeddings generated by it as a stand-alone feature.
With regards to single embedding models, we make similar observations as for the retrieval task: single embedding \our infers a better embedding.
Amongst \our variants, we note that the ranking task is less sensitive to embedding computation and hence centroid, medoid and sequence models have similar performances as the embeddings are only used to order pins.
However it is sensitive to cluster importance scores as that determines which $3$ user embeddings are picked for ranking.

\begin{table}[t]
\caption{Lift relative to \textit{last pin model} for ranking task.}
\label{tab:rankingtask}
\begin{tabular}[t]{l|c|c}
& \textit{\textbf{R-Prec.}} & \textit{\textbf{Rec. Rank}} \\ \hline
Last pin model & 0\% & 0\% \\
Decay avg. model ($\lambda = 0.01$ & 8\% & 7\% \\
Sequence models (HierTCN) & 21\% & 16\% \\
\our(sample $1$ embedding) & 24\% & 18\% \\ \hline
\our(Kmeans(k=5)) & 32\% & 24\% \\
\our(Complete Linkage) & 29\% & 22\% \\
\our(embedding = Centroid) & \textbf{37\%} & \textbf{28\%} \\
\our(embedding = HierTCN) & \textbf{37\%} & \textbf{28\%} \\
\our(importance $\lambda=0$) & 31\% & 24\% \\
\our(importance $\lambda=0.1$) & 30\% & 24\% \\
\our (Ward, Medoid, $\lambda=0.01$) & \textbf{37\%} & \textbf{28\%}
\end{tabular}
\end{table}

\begin{figure}[b]
\includegraphics[scale=0.3]{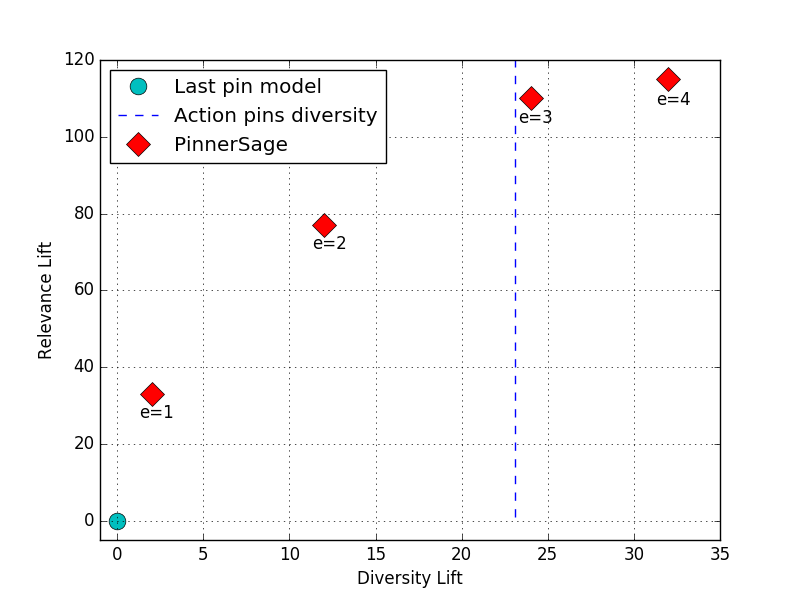}
\caption{Diversity relevance tradeoff when different number of embeddings ($e$) are selected for candidate retrieval.}
\label{fig:divrel}
\end{figure}

\subsubsection{\textbf{Diversity Relevance Tradeoff.}}
Recommender systems often have to trade between relevance and diversity~\cite{div}.
This is particularly true for single embedding models that have limited focus.
On the other-hand a multi-embedding model offers flexibility of covering disparate user interests simultaneously.
We define diversity to be average pair-wise \textit{cosine distance} between the recommended set.

Figure~\ref{fig:divrel} shows the diversity/relevance lift w.r.t. last pin model.
We note that by increasing $e$, we increase both relevance and diversity.
This intuitively makes sense as for larger $e$, the recommendation set is composed of relevant pins that span multiple interests of the user.
For $e > 3$ the relevance gains tapers off as users activities do not vary wildly in a given day (on average). 
Infact, for $e=3$, the recommendation diversity achieved by \our matches closely with action pins diversity, which we consider as a sweet-spot.

\section{Related Work}
\label{section:related_work}
There is an extensive research work focused towards learning embeddings for users and items, for e.g., \cite{baba2019embedding_retrieval, dlyoutube, widedeep, starspace, htcn}.
Much of this work is fueled by the models proposed to learn word representations, such as Word2Vec model~\cite{word2vec} that is a highly scalable continuous bag-of-words (CBOW) and skip-gram (SG) language models.
Researchers from several domains have adopted word representation learning models for several problems such as for recommendation candidate ranking in various settings, for example for movie, music, job, pin recommendations~\cite{music, job, movie, htcn}.
There are some recent papers that have also focused on candidate retrieval. \cite{widedeep} mentioned that the candidate retrieval can be handled by a combination of machine-learned models and human-defined rules; 
\cite{baba2019embedding_retrieval} considers large scale candidate generation from billions of users and items, and proposed a solution that pre-computes hundreds of similar items for each embedding offline. 
\cite{dlyoutube} has discussed a candidate generation and retrieval system in production based on a single user embedding.


Several prior works \cite{microprofile, westonmultiintr, alibabaKdd2019} consider modeling users with multiple embeddings. \cite{microprofile} uses multiple time-sensitive contextual profile to capture user's changing interests. \cite{westonmultiintr} considers max function based non-linearity in factorization model, equivalently uses multiple vectors to represent a single user, and shows an improvement in 25\% improvement in YouTube recommendations. \cite{alibabaKdd2019} uses polysemous embeddings (embeddings with multiple meanings) to improve node representation, but it relies on an estimation of the occurrence probability of each embedding for inference. Both \cite{westonmultiintr, alibabaKdd2019} report results on offline evaluation datasets.
Our work complements prior work and builds upon them to show to operationalize a rich multi-embedding model in a production setting.

\section{Conclusion}
In this work, we presented an end-to-end system, called PinnerSage, that powers personalized recommendation at Pinterest. 
In contrast to prior production systems that are based on a single embedding based user representation, \our proposes a multi-embedding based user representation scheme. 
Our proposed clustering scheme ensures that we get full insight into the needs of a user and understand them better.
To make this happen, we adopt several design choices that allows our system to run efficiently and effectively, such as medoid based cluster representation and importance sampling of medoids.
Our offline experiments show that our approach leads to significant relevance gains for the \textit{retrieval task}, as well as delivers improvement in \textit{reciprocal rank} for the \textit{ranking task}.
Our large A/B tests show that \our provides significant real world online gains.
Much of the improvements delivered by our model can be attributed to its better understanding of user interests and its quick response to their needs.
There are several promising areas that we consider for future work, such as selection of multiple medoids per clusters and a more systematic reward based framework to incorporate implicit feedback in estimating cluster importance.


\section{Acknowledgements}
We would like to extend our appreciation to Homefeed and Shopping teams for helping in setting up online A/B experiments.
Our special thanks to the embedding infrastructure team for powering embedding nearest neighbor search.

%
\bibliographystyle{abbrv}
\bibliography{references}
%
%



\newpage
\section*{Reproducibility Supplementary Materials}
\subsection*{APPENDIX A: Convergence proof of Ward clustering algorithm}
\label{section:appendix}
\begin{lemma}
In Algo.~\ref{algo:WardClustering}, a merged cluster $C_i \cup C_j$ cannot have distance lower to another cluster $C_k$ than the lowest distance of its children clusters to that cluster, i.e., 
$d(C_i \cup C_j, C_k) \geq \min \{ d_{ik}, d_{jk} \}$.
\label{lem:dist}
\end{lemma}
\begin{proof}[Proof]
For clusters $i$ and $j$ to merge, the following two conditions must be met: $d_{ij} \leq d_{ik}$ and $d_{ij} \leq d_{jk}$.
Without loss of generality, let $d_{ij} = d$ and $d_{ik} = d + \gamma$ and $d_{jk} = d + \gamma + \delta$, where $\gamma \geq 0, \delta \geq 0$.
We can simplify eq.~\ref{eq:wardminvar} as follows:
\begin{align}
d(C_i \cup C_j, C_k) &= \frac{(n_i + n_k) \ (d + \gamma) + (n_j + n_k) \ (d+\gamma + \delta) - n_k \ d}{n_i + n_j + n_k} \nonumber \\
&= d + \gamma + \frac{n_k \ \gamma + (n_j + n_k) \ \delta}{n_i+n_j+n_k} \geq d + \gamma
\label{eq:proofincrward}
\end{align}
$d(C_i \cup C_j, C_k) \geq d + \gamma$ implies $d(C_i \cup C_j, C_k) \geq \min\{d_{ik}, d_{jk}\}$.
\end{proof}

\begin{lemma}
A cluster cannot be added twice to the stack in Ward clustering (algo.~\ref{algo:WardClustering}).\label{lem:ward}
\end{lemma}
\begin{proof}[Proof by contradiction]
Let the state of stack at a particular time be $[ \ldots, i, j, k, i ]$.
Since $j$ is added after $i$ in stack, this implies that $d_{ij} \leq d_{ik}$ (condition 1).
Similarly from subsequent additions to the stack, we get $d_{jk} \leq d_{ji}$ (condition 2) and $d_{ki} \leq d_{kj}$ (condition 3).
We also note that by symmetry $d_{xy} = d_{yx}$.
Combining condition 2 and 3 leads to $d_{ik} \leq d_{ij}$, which would contradict condition 1 unless $d_{ji} = d_{jk}$.
Since $i$ is the second element in the stack after addition of $j$, $j$ cannot add $k$ given $d_{ji} = d_{jk}$.
Hence $i$ cannot be added twice to the stack. 

Additionally, a merger of clusters since the first addition of $i$ to the stack, cannot add $i$ again.
This is because its distance to $i$ is greater than or equal to the smallest distance of its child clusters to $i$ by Lemma~\ref{lem:dist}.
Since the child cluster closest to $i$ cannot add $i$, so can't the merged cluster.
\end{proof}

\end{document}